\documentclass[letterpaper, 10 pt, conference]{ieeeconf}
\IEEEoverridecommandlockouts
\usepackage{times}
\usepackage{multicol}
\usepackage{multirow}
\usepackage[bookmarks=true]{hyperref}
\usepackage{amsmath, amssymb, mathrsfs,MnSymbol}

\usepackage[ruled,vlined,linesnumbered]{algorithm2e}
\newtheorem{problem}{Problem}
\newtheorem{example}{Example}

\newtheorem{theorem}{Theorem}[section]
\newtheorem{lemma}[theorem]{Lemma}

\usepackage{graphicx} 
\usepackage[colorinlistoftodos,prependcaption,textsize=tiny]{todonotes}
\usepackage{tikz}
\usetikzlibrary{positioning,automata}
\usepackage{clock}
\usepackage[clock]{ifsym}
\usepackage{xcolor}

\usepackage{comment}

\makeatletter
\newcommand\footnoteref[1]{\protected@xdef\@thefnmark{\ref{#1}}\@footnotemark}
\makeatother
\allowdisplaybreaks

\begin{document}

\title{\LARGE \bf Robust MITL planning under uncertain navigation times}

\author{Alexis Linard*, Anna Gautier*, Daniel Duberg, Jana Tumova
\thanks{*contributed equally. The authors are with the KTH Royal Institute of Technology, SE-100 44, Stockholm, Sweden, with the Division of Robotics, Perception and Learning. This work was supported by the Wallenberg AI, Autonomous Systems and Software Program (WASP) funded by the Knut and Alice Wallenberg Foundation, Digital Futures, and carried out as part of the Vinnova Competence Center for Trustworthy Edge Computing Systems and Applications at KTH Royal Institute of Technology. 
{\tt\small \{linard,annagau,dduberg,tumova\}@kth.se}}%
}

\maketitle

\begin{abstract}
In environments like offices, the duration of a robot's navigation between two locations may vary over time. For instance, reaching a kitchen may take more time during lunchtime since the corridors are crowded with people heading the same way. In this work, we address the problem of routing in such environments with tasks expressed in Metric Interval Temporal Logic (MITL) -- a rich robot task specification language that allows us to capture explicit time requirements. 
Our objective is to find a strategy that maximizes the temporal robustness of the robot's MITL task. 
As the first step towards a solution, we define a Mixed-integer linear programming approach to solving the task planning problem over a Varying Weighted Transition System, where navigation durations are deterministic but vary depending on the time of day. Then, we apply this planner to optimize for MITL temporal robustness in Markov Decision Processes, where the navigation durations between physical locations are uncertain, but the time-dependent distribution over possible delays is known. Finally, we develop a receding horizon planner for Markov Decision Processes that preserves guarantees over MITL temporal robustness. We show the scalability of our planning algorithms in simulations of robotic tasks.
\end{abstract}

\begin{keywords}
Formal Methods, Planning Under Uncertainty, Temporal Robustness, Markov Decision Processes.
\end{keywords}

\IEEEpeerreviewmaketitle

\section{Introduction}
\label{sec:intro}

We consider a scenario where a robot in an office-like environment receives various tasks over time (such as repeatedly visiting offices A, B and C three times every day, while recharging at least every 2 hours) and needs to be routed to complete them. 
We express tasks in Metric Interval Temporal Logic (MITL) and associate them with a priority. MITL is an extension of Linear Temporal Logic (LTL) where time intervals equip the temporal operators \cite{koymans1990specifying}, and where the upper bound of the time intervals expresses the deadlines of the tasks. Both LTL and its timed version have been recently popular choices of robot task and motion specification language due to their rigorousness and richness \cite{fainekos_2005,kressgazit_2009,lahijanian_2012,raman_2014, vasile2017minimum, liang2022fair, cardona2022planning}.
Since MITL tasks include time constraints, the time it takes for the robot to navigate between locations will impact whether they are satisfied or not. 
\emph{Temporal robustness} -- a quantitative semantics originally defined for Signal Temporal Logic -- measures the degree to which a strategy satisfies a specification when subjected to time shifts \cite{rodionova2022temporal}. In other words, it indicates \textit{how much delay a strategy can afford} while still satisfying the desired tasks.

In this paper, we consider stochastic navigation duration, where the time to move between two physical locations follows a known distribution, which, however, may change over time.
For instance, if a robot has to enter the kitchen during lunchtime, it might take extra time with some positive probability. 
In order to successfully optimize for MITL temporal robustness under time-dependent uncertainty, as a first step, we consider the case where navigation is deterministic, but time-dependent. In this setting, a robot navigates an environment abstracted as a discrete Varying Weighted Transition System (VWTS).
The states of the VWTS represent the physical locations in the environment, while the weighted transitions represent the time it takes to transit between two states. Unlike \cite{kurtz2021more}, these navigation times may vary depending on the time of day. Our goal is to find a \textit{strategy}, i.e., a sequence of states in the VWTS, that satisfies MITL tasks with the least possible delays while accounting for known dynamic travel times between states. 

To address the strategy synthesis for stochastically varying navigation times, we model the problem as a Markov Decision Process (MDP), where the transition function incorporates the uncertainty. The goal is to synthesize a strategy that maximizes the tasks' expected temporal robustness. Towards this, we design a reward function that optimizes for temporal robustness. Because solving the MDP exactly can be costly, we develop a receding horizon planner for MDPs. Our receding horizon planner uses the timed VWTS model as a worst-case lookahead, thus preserving guarantees over MITL temporal robustness.

Related work includes \cite{tumova2016least} and \cite{kurtz2021more}, which, however, makes assumptions on static navigation duration.
In \cite{tumova2016least}, a vehicle's trajectory must meet all the temporal demands within their respective deadlines. In contrast to our work, navigation times were time-independent and deterministic, and demands were limited to the syntactically co-safe fragment of LTL. \cite{kurtz2021more} proposed a scalable mixed-integer encoding for a WTS with finite, deterministic, time-independent weights, subject to MITL specifications. There, the starting times and deadlines of the tasks were directly encoded in the intervals of the temporal operators.
Finally, \cite{rodionova2022temporal} defines temporal robustness for Signal Temporal Logic specifications, and requires minor adjustments to handle MITL.

Our contributions are -- (A) A Mixed Integer Linear Program encoding of a VWTS with time-varying weights, thereby \emph{integrating temporal task robustness expressed in MITL} in the optimization function. (B)~A method for \emph{optimizing MITL robustness in MDPs}, accommodating scenarios where access to a specific spatial location at a particular time may follow a probabilistic distribution rather than a fixed time value. (C) An \emph{evaluation on the scalability of our planning algorithms} in simulations of robotic tasks.

\section{Notation}
\label{sec:notations}

Let $\mathbb{R}$, $\mathbb{Z}$ and $\mathbb{N}$ be the set of real, integer and natural numbers including zero, respectively. 
We use a discrete notion of time throughout this paper, and time intervals are represented by $[a,b] \subset \mathbb{N}$, $a, b \in \mathbb{N}, a < b$. 

The syntax of MITL is defined as \cite{koymans1990specifying}:
\begin{equation*}\label{eq:mtl_syntax}
   \phi := \top \mid \pi \mid \lnot \phi \mid \phi_1 \land \phi_2 \mid \phi_1 \mathcal{U}_{[a,b]} \phi_2 
\end{equation*}
where $\pi \in \Pi$ is an atomic proposition; $\neg$ and $\wedge$ are the Boolean operators for negation and conjunction, respectively; and $\mathcal{U}_{[a,b]}$ is the temporal operator {\em until} over bounded interval $[a,b]$.
Other Boolean operations are defined using the conjunction and negation operators to enable the full expression of propositional logic.
Additional temporal operators {\em eventually} and {\em always} are defined as $\diamondsuit_{[a,b]} \phi \equiv \top \mathcal{U}_{[a,b]} \phi$ and $\Box_{[a,b]} \phi \equiv \neg \diamondsuit_{[a,b]} \neg \phi$,  respectively. 

We define MITL semantics over \textit{timed words} $\mathbf{\sigma} = (\xi_0,t_0),(\xi_1,t_1),(\xi_2,t_2),\dots$, $\sigma_{t_i} = (\xi_i,t_i)$, where $\xi_i \in 2^{\Pi}$ is the set of atomic propositions that hold at position $i \in \mathbb{N}$ and throughout time interval $[t_i, t_{i+1}]$, $t_i \in \mathbb{N}, \forall i\geq 0$. We denote that $\mathbf{\sigma}$ satisfies the MITL formula $\phi$ with $\mathbf{\sigma} \models \phi$, and that $\sigma_i,\sigma_{i+1},\dots$ satisfies $\phi$ with $\mathbf{\sigma}_i \models \phi$. 
We refer to the set of timed words as $\Sigma$.

The characteristic function $\chi_\phi(\sigma) : \Sigma \rightarrow \{\pm 1\}$ indicates if a formula is satisfied \cite{donze_robust_stl}.
\begin{equation*}
    \label{eq:characteristic_func}
    \begin{aligned}
        \chi_{\pi}(\sigma_{t_i}) & = \begin{cases} 1 & \textrm{if}\ \pi \in \xi_i \\ -1 & \textrm{otherwise} \end{cases}
        \\
        \chi_{\neg \phi}(\sigma_{t_i}) & = -\chi_\phi(\sigma_{t_i})
        \\
        \chi_{\phi_1 \land \phi_2}(\sigma_{t_i}) & = \chi_{\phi_1}(\sigma_{t_i}) \ \wedge\ \chi_{\phi_2}(\sigma_{t_i})
        \\
        \chi_{\Box_{[a,b]}\phi}(\sigma_{t_i}) & = \bigwedge_{t'\in {t_i}+[a,b]} \chi_{\phi}(\sigma_{t'})
        \\
        \chi_{\diamondsuit_{[a,b]}\phi}(\sigma_{t_i}) & = \bigvee_{t'\in {t_i}+[a,b]} \chi_{\phi}(\sigma_{t'})
        \\
        \chi_{\phi_1 \mathcal{U}_{[a,b]} \phi_2}(\sigma_{t_i}) & = 
        \bigvee_{t'\in {t_i}+[a,b]} \left( \chi_{\phi_2}(\sigma_{t'}) \ \wedge \ \bigwedge_{t'' \in [{t_i},t']} \chi_{\phi_1}(\sigma_{t''}) \right)
        \\
        \chi_{\phi}(\sigma) & = \chi_{\phi}(\sigma_{t_0})
    \end{aligned}
\end{equation*}

Adapting \cite{rodionova2022temporal}, we define the (synchronous) temporal robustness $\eta : \Sigma \rightarrow \mathbb{Z}$ of an MITL formula $\phi$ on a word $\sigma$.
\begin{equation*}
    \label{eq:time_robustness}
    \begin{aligned}
    \eta_\phi^-(\sigma_{t_i}) & = \chi_{\phi}(\sigma_{t_i}) \cdot \sup\{\tau \geq 0 : \forall t' \in [{t_i}-\tau,{t_i}], \\ & \chi_{\phi}(\sigma_{t_i}) = \chi_{\phi}(\sigma_{t'})\} 
    \\
    \eta_\phi^+(\sigma_{t_i}) & = \chi_{\phi}(\sigma_{t_i}) \cdot \sup\{\tau \geq 0 : \forall t' \in [{t_i},{t_i}+\tau], \\ & \chi_{\phi}(\sigma_{t_i}) = \chi_{\phi}(\sigma_{t'})\} 
    \\
    \eta_\phi^{\pm}(\sigma_{t_i}) & = \chi_{\phi}(\sigma_{t_i}) \cdot \sup\{\tau \geq 0 : \forall t' \in {t_i} \pm [0,\tau], \\ &  \chi_{\phi}(\sigma_{t_i}) = \chi_{\phi}(\sigma_{t'})\} 
    \end{aligned}
\end{equation*}
Intuitively, temporal robustness measures how well the satisfaction of a formula $\phi$ holds with respect to time shifts. It quantifies the maximal amount of time that we can shift the characteristic function $\chi_{\phi}(\sigma_{t_i})$ of $\phi$ to the left ($\eta_\phi^-(\sigma_{t_i})$) or right ($\eta_\phi^+(\sigma_{t_i})$) without changing the value of $\chi_{\phi}(\sigma_{t_i})$. 
The combined notation $\eta_\phi^{\pm}(\sigma_{t_i})$ quantifies the maximal amount of time by which we can shift $\chi_{\phi}(\sigma_{t_i})$ to the left and to the right, without changing the value of $\chi_{\phi}(\sigma_{t_i})$.

\section{Strategy Synthesis under Deterministic Navigation Times}
\label{sec:problem-deterministic}

We consider strategy synthesis for finite-state labelled Weighted Transition Systems (WTS) as a step towards handling uncertainties in navigation times. 
A WTS $(S, s_0, \delta, \Pi, L, C)$ is a tuple where $S$ is a finite set of states, $s_0 \in S$ is an initial state, $\delta \subseteq S \times S$ are transition relations, $\Pi$ is a finite set of atomic predicates, $L : S \to 2^{\Pi}$ is a labeling function associating states to atomic predicates and $C$ is a weight function assigning weights to each of the transitions. The weight $C(s_1, s_2)$ represents the number of time steps it takes for a robot to navigate from $s_1$ to $s_2$.
Further, we define the set of adjacent states of state $s \in S$ as $Adj(s) = \{ s' \in S \mid (s,s') \in \delta \}$.

For example, WTS can represent an abstraction of an office-like environment. The set of states $S$ represent waypoints in the environment (e.g., a particular location in an office, entrance to the office, etc.); $\delta$ represents the connection between locations navigable by the robot; $L$ is the set of labels of the office-like environment (e.g., ``\textit{kitchen}'', ``\textit{lab}'', ``\textit{office}'') and $C$ the nominal number of time steps it takes for the mobile robot to navigate between waypoints.

We now generalize the classic definition of a WTS and introduce the Varying Weighted Transition Systems (VWTS). A VWTS is defined by $\mathcal{A}=(S, s_0, \delta, \Pi, L, \Delta)$, where $\Delta : (S \times \mathbb{N} \times S) \to \mathbb{N}^+$ is a \textit{time-dependent weight function}. The weight $\Delta(s_1, t,  s_2)$ represents the \textit{number of time steps} it takes for a robot to navigate from $s_1$ to $s_2$ at timestep $t$.
In Fig. \ref{fig:example_wts}, we show an example of a VWTS. 
In this work, we want to find a \textit{strategy}, that is, the sequence of states of a VWTS $\mathcal{A}$ that satisfies a given MITL specification. Further, if no satisfying solution exists, we want to find a solution that minimizes the \textit{delay} or the \textit{advancement} of given tasks in $\phi$. In the following, we reuse some notations from \cite{kurtz2021more}.

\begin{figure}[t] 
    \centering
    \hspace{-1cm}
    \begin{minipage}[b]{0.45\linewidth}
        
        \includegraphics[width=0.95\linewidth,angle=-90]{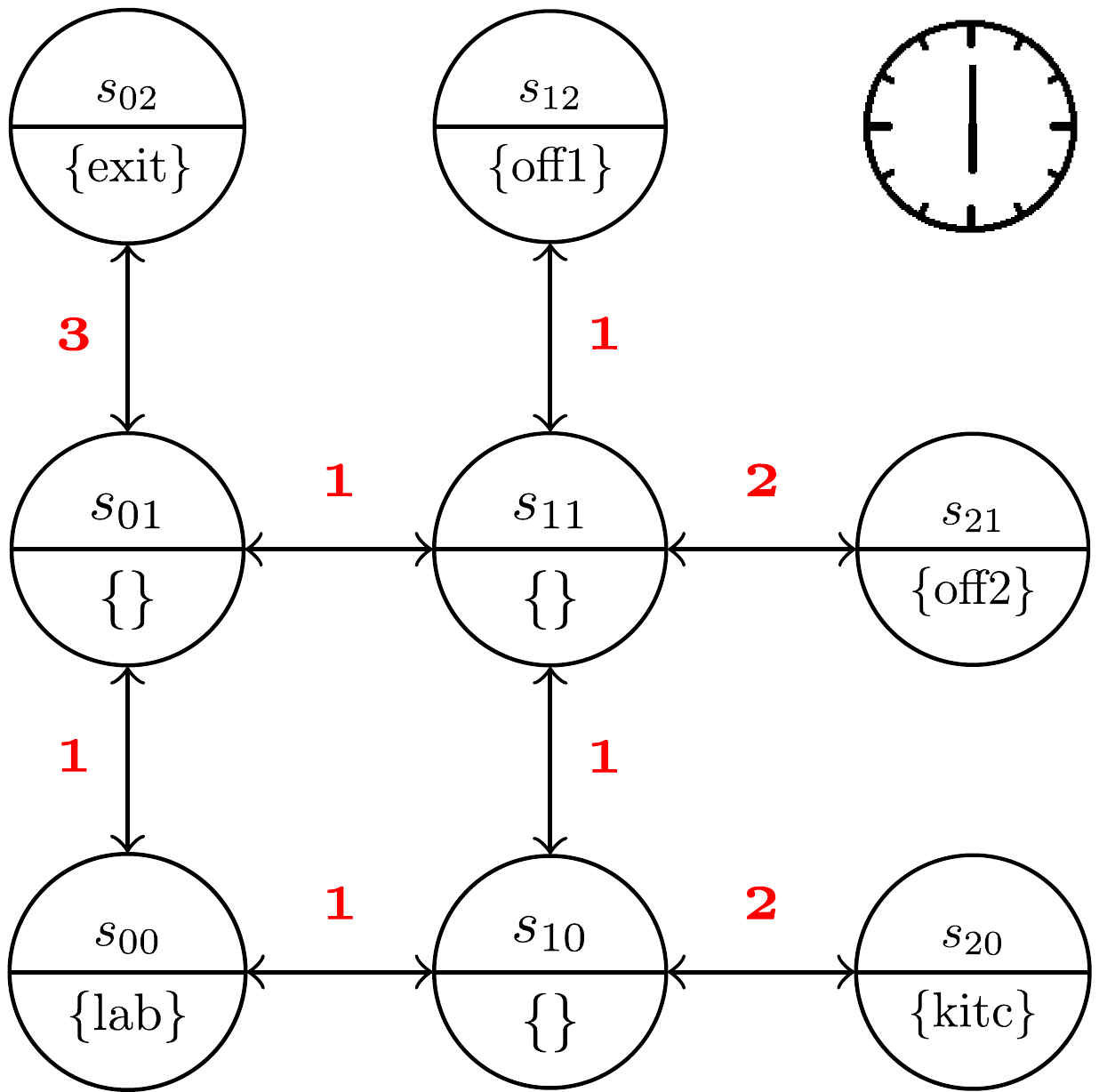}
    \end{minipage}
    \hspace{0.5cm}
    \begin{minipage}[b]{0.45\linewidth}
        \includegraphics[width=0.95\linewidth,angle=-90]{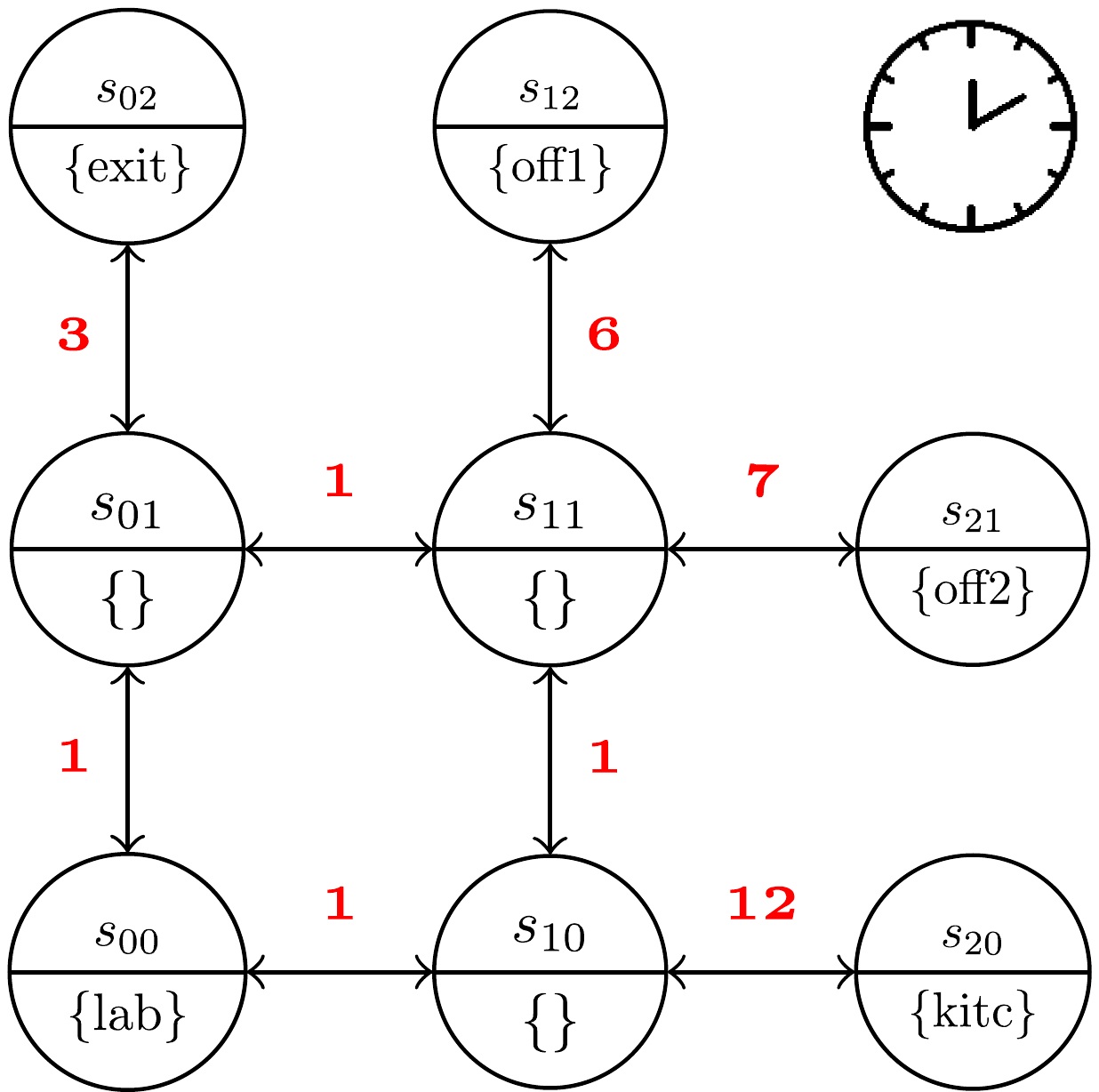}
    \end{minipage}
  \caption{VWTS  $\mathcal{A} = (S, s_0, \delta, \Pi, L, \Delta)$ shows a simplified office-like environment. States are labelled by $L=\{``\textit{exit}",``\textit{off1}",``\textit{off2}",``\textit{lab}",``\textit{kitc}"\}$. Self-transitions are not shown but are assumed to have a weight of 1, and transition weights are bidirectional. Transition weights are shown along the edges, and change based on time.}
  \label{fig:example_wts} 
\end{figure}

A sequence of states $\mathbf{s} = s_0,s_1,\dots,s_n$ is a \textit{path} of the VWTS $\mathcal{A}$ if $(s_i,s_{i+1}) \in \delta \ \forall i \in [0,n-1]$.
Each \textit{path} $\mathbf{s}$ is associated with a \textit{time sequence} $\mathbb{T}(\mathbf{s}) = t_0, t_1,\dots,t_n$ where $t_0 = 0$, $t_i = t_{i-1} + \Delta(s_{i-1},t_{i-1},s_i)$ for $i \geq 1$, and $t_n \leq T$, where $T$ is the planning horizon. The time $t_i$ denotes the sum of the weights of the transitions executed, that is, the time elapsed until reaching the $i$-th state in the \textit{path} $\mathbf{s}$.
Further, we also associate to any path $\mathbf{s}$ a \textit{timed word} $\sigma(\mathbf{s}) = (L(s_0),t_0),(L(s_1),t_1),(L(s_2),t_2),\dots$.
In other words, $\sigma(\mathbf{s})$ is the corresponding timed sequence of atomic propositions of path $\mathbf{s}$. This \textit{timed word} is then considered to evaluate the satisfaction of an MITL formula. 

\begin{example}
\label{example:wts}
Consider the right VWTS in Fig. \ref{fig:example_wts}. The path $\mathbf{s} = s_{02},s_{01},s_{00},s_{10},s_{11},s_{12}\ldots$ yields the time sequence $\mathbb{T}(\mathbf{s}) = 0, 3, 4, 5, 6, 12\ldots$, which results in the timed word $\sigma(\mathbf{s}) = (\{``\textit{exit}"\},0),(\emptyset,3),(\{``\textit{lab}"\},4),(\emptyset,5),(\emptyset,6),(\{``\textit{off1}"\},$ $12)\ldots$ where $\{``\textit{exit}"\}$ is the label of $s_{02}$ at $t_0=0$, $t_1=3$, \ldots.  
Note that until the next state is reached, the label of the previous state still holds. 
Let us look at the specification $\Box_{1,2}``\textit{exit}"$. 
Since at times 1 and 2 the agent is transiting between states $s_{02}$ and $s_{01}$, under our definition, the label of $s_{02}$ holds until $s_{01}$ is reached. Hence, $\Box_{1,2}``\textit{exit}"$ holds. \end{example}

Now, consider that a robot navigating an arbitrary VWTS $\mathcal{A}$ is given a set of tasks $D = \{(\phi_1,p_1),(\phi_2,p_2),\dots\}$, where $\phi_i$ is an arbitrary MITL formula and $p_i$ its assigned priority. Priority controls the emphasis given to differing, and potentially conflicting, goals.  We, therefore, define our problem as follows:

\begin{problem}
\label{prob:main_problem_WTS}
Given a VWTS $\mathcal{A}$, a set of tasks $D$ expressed in MITL, a planning horizon $T > ||D||$, where $||D||$ is the tasks' maximum horizon, we would like to find a strategy, i.e., path, that maximizes the total sum of the tasks' temporal robustness, weighted by their respective priorities, i.e.,
   \begin{subequations}\label{eq:main_problem}
    \begin{align}
        \max_{\mathbf{s}} ~& \sum_{(\phi_i,p_i) \in D} \eta_{\phi_i}(\sigma(\mathbf{s})) \cdot p_i \label{eq:main_problem_cost}\\
        \mathrm{s.t.~} & \text{Transition System Constraints} \label{eq:main_ts_constraint} \\
        & \text{MITL Constraints} \label{eq:main_mtl_constraint} 
    \end{align}
    \end{subequations}
\end{problem}

We omit specifying what specific temporal robustness ($\eta_\phi^+$, $\eta_\phi^-$ or $\eta_\phi^{\pm}$) to use in the optimization function. Depending on the application, one might prefer optimizing one or the other. More generally, the optimization function could be any linear combination of functions.

\subsection{Approach}

\subsubsection{Encoding of the VWTS \eqref{eq:main_ts_constraint}}

We start by introducing a binary variable $b_s^t$ for each state $s \in S$ and each timestep $t \in [0,T]$, where $T$ is the planning horizon. 
If $b_s^t = 1$, then the optimal path visits state $s$ at time $t$.
The timed word resulting from the optimal path is then retrieved by $\sigma(\mathbf{s})=(L(s),t)\ \forall b_s^t$ s.t $b_s^t = 1$.
We define the constraints encoding the VWTS $\mathcal{A} = (S, s_0, \delta, \Pi, L, \Delta)$:
\begin{subequations}\label{eq:transition_system_constraints}
\begin{gather}
    b_{s_0}^0 = 1, \label{eq:standard_initial_constraint} \\
     0 \leq\sum_{s \in S} b_s^t \leq 1 \quad \forall t \in [0,T], \label{eq:standard_occupation_constraint}\\
    \sum_{s' \in Adj(s)} b_{s'}^{t+\Delta(s,t,s')} \geq b_s^t \quad \forall s \in S, \forall t \in [0,T],\label{eq:standard_transition_constraint} \\
    \sum_{s' \in Adj(s)} b_{s'}^{t+\Delta(s,t,s')} \leq 1 \ \quad \forall s \in S, \forall t \in [0,T]\label{eq:standard_transition_constraint2}.
\end{gather}
\end{subequations}
Constraint \eqref{eq:standard_initial_constraint} establishes $s_0$ at $t=0$; \eqref{eq:standard_occupation_constraint} ensures that 0 or 1 state can be occupied at each timestep; and  \eqref{eq:standard_transition_constraint} and \eqref{eq:standard_transition_constraint2} enforce transition relations. 
While \eqref{eq:standard_initial_constraint} follows the encoding presented in \cite{kurtz2021more}, note that \eqref{eq:standard_occupation_constraint} differs to account for travel times that are different from 1, and that 0 or 1 state can be occupied at each timestep. Further, \eqref{eq:standard_transition_constraint} and \eqref{eq:standard_transition_constraint2} differ from \cite{kurtz2021more} in that they establish transition relations with dynamic time weights over time.

\subsubsection{Encoding of MITL constraints \eqref{eq:main_mtl_constraint}}

We recursively define variables and constraints along the MITL formula's structure. We define new binary variables, $z_\varphi^t$, such that $z_\varphi^t = 1$ if and only if $\varphi$ is satisfied starting from time $t$. 
Conjunction and disjunction are encoded as follows \cite{kurtz2021more}:
\begin{gather}
    z \vDash \bigwedge_{i=1}^n z_i \iff z \leq z_i ~\forall i \text{ and } z \geq 1 - n + \sum_{i=1}^n z_i,  \label{eq:conj} \\
    z \vDash \bigvee_{i=1}^n z_i \iff z \leq \sum_{i=1}^n z_i \text{ and } z \geq z_i ~\forall i. \label{eq:disj}
\end{gather}
We use \eqref{eq:conj} and \eqref{eq:disj} to encode an MITL formula $\varphi$:
\begin{subequations}\label{eq:mtl_constraints}
\begin{alignat}{2}
    & \pi                       && \implies z_\pi = \sum_{s \in S \mid L(s) = \pi} b_s^{\tilde{q}^t}, \label{eq:modified_AP_encoding} \\
    & \lnot \varphi             && \implies z_{\lnot \varphi} = 1 - z_{\varphi}, \\
    & \varphi_1 \land \varphi_2 && \implies z_{\varphi_1 \land \varphi_2} = z_{\varphi_1} \land z_{\varphi_2}, \\
    & \varphi_1 \lor \varphi_2 && \implies z_{\varphi_1 \lor \varphi_2} = z_{\varphi_1} \lor z_{\varphi_2}, \\
    & \Box_{[a,b]}\varphi && \implies \bigwedge_{t' \in [t+a,t+b]} z_{\varphi}^{t'}, \\
    & \diamondsuit_{[a,b]}\varphi && \implies \bigvee_{t' \in [t+a,t+b]} z_{\varphi}^{t'}, \\ 
    & \varphi_1\mathcal{U}_{[a,b]}\varphi_2 && \implies \\
    & &&\bigvee_{t' \in [t+a,t+b]} \left(z_{\varphi_2}^{t'} \land \bigwedge_{t'' \in [t,t'-1]} z_{\varphi_1}^{t''} \right). \nonumber
\end{alignat}
\end{subequations}
where $\tilde{q}^t$ is an integer variable referring to the last time step where $b_s^{t'} = 1$ holds :
\begin{subequations}
\begin{gather}
    q^t = (q^{t-1}+1) \cdot (1 - \sum_{s \in S} b_s^t) \ \forall t \in [0,T] \label{eq:qt} \\
    \tilde{q}^t = t - q^t \label{eq:qt_tilde} 
\end{gather}
\end{subequations}
where $q^t$ is a counter variable counting for how many time steps $b_s^{t'} = 0$ holds from time $t$.
Note that the definition of \eqref{eq:modified_AP_encoding} and \eqref{eq:qt} enable the evaluation of atomic predicates when no observation can be made while transiting between 2 states, as remarked in Example \ref{example:wts}. 

\subsubsection{Encoding of MITL temporal robustness \eqref{eq:main_problem_cost}}

We follow the encoding of constraints for the left and right temporal robustnesses $\eta_\phi^+$ and $\eta_\phi^-$ presented in \cite{rodionova2022temporal}.
A \textit{positive} value of $\eta_\phi^+$ stands for how many time units $\sigma(\mathbf{s})$ can be advanced (or shifted to the left) while maintaining the satisfaction of $\phi$. Conversely, a \textit{negative} value of $\eta_\phi^+$ stands for how many time units $\sigma(\mathbf{s})$ should be advanced (or shifted to the left) to obtain the satisfaction of $\phi$, in other words, a \textit{negative} value of $\eta_\phi^+$ stands for the delay in which $\phi$ is satisfied.
Concerning the right temporal robustness, a \textit{positive} value of $\eta_\phi^-$ stands for how many time units $\sigma(\mathbf{s})$ can be postponed (or shifted to the right) while maintaining the satisfaction of $\phi$. In other words, a \textit{positive} value of $\eta_\phi^-$ stands for how much time one can afford to postpone completion of $\phi$.

Since the encoding of the left temporal robustness has been explicitly defined in \cite{rodionova2022temporal}, we solely define the encoding of the right temporal robustness in the following, and refer the reader to \cite{rodionova2022temporal} or our implementation for the left temporal robustness.
If $z_\varphi^t=1$, we count the maximum number of sequential time points $t'<t$ in the past for which $z^{t'}_\varphi=1$. If $z^t_\varphi=0$, we then want to count the maximum number of sequential time points $t'<t$ in the past for which $z^{t'}_\varphi=0$, and then multiply this number with $-1$. To encode the right temporal robustness, we adapt the encoding proposed in \cite{rodionova2022temporal} for the left temporal robustness, and introduce counter variables $c'^t_{1,\varphi}$, $c'^t_{0,\varphi}$, $\tilde{c}'^t_{1,\varphi}$ and $\tilde{c}'^t_{0,\varphi}$, for all $t \in [-T',0]$, where $T'$ is a user-defined constant standing for the maximum right temporal robustness to be calculated.
Note that the encoding of the right temporal robustness also requires defining fictive states $b_{s}^t$ in negative time steps ranging from $[-T',0]$, with:
\begin{subequations}\label{eq:extra_right_time_ro_constraints}
\begin{gather}
    b_{s}^t = 0 \quad \forall s \in S, \forall t \in [-T',0],\label{eq:extra_right_time_ro_constraints1} \\
    z_{\phi_i}^t \quad \forall (\phi_i,p_i) \in D, \forall t \in [0,T]\label{eq:extra_right_time_ro_constraints2}.
\end{gather}
\end{subequations}
Constraints for the right temporal robustness are defined: 
\begin{subequations}\label{eq:eta_minus}
\begin{align}
	\label{eq:milp_c1_minus}
	& c'^t_{1,\varphi} \iff (c'^{t-1}_{1,\varphi} +1) \cdot z_\varphi^t, \qquad\ \,\qquad c'^{-T'-1}_{1,\varphi} = 0, \\
	\label{eq:milp_c0_minus}
	& c'^t_{0,\varphi}  \iff (c'^{t-1}_{0,\varphi} -1) \cdot (1-z_\varphi^t), \qquad c'^{-T'-1}_{0,\varphi} = 0, \\
    \label{eq:milp_ctp1_minus}
    & \tilde{c}'^t_{1,\varphi} \iff c'^t_{1,\varphi} - z_\varphi^t,\\
    \label{eq:milp_ctp0_minus}
    & \tilde{c}'^t_{0,\varphi} \iff c'^t_{0,\varphi} + (1-z_\varphi^t), \\
	\label{eq:milp_eta_final_minus}
    &\eta_\phi^-(\sigma(\mathbf{s})) \iff \tilde{c}'^0_{1,\varphi} + \tilde{c}'^0_{0,\varphi}.
 \end{align}
\end{subequations}

\begin{figure}[t] 
	\centering
	\includegraphics[width=0.95\linewidth]{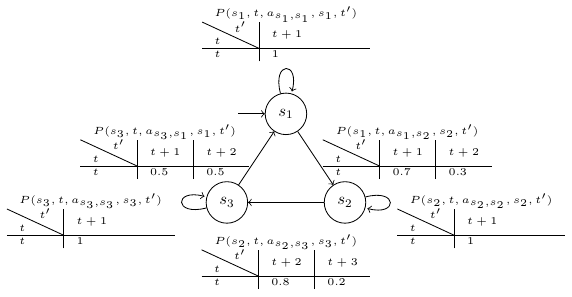}
	\caption{MDP $\mathcal{M}=(S, A, P, T, \Pi, L)$ is simplified example with 3 states.  Probabilistic transitions are shown by the tables next to each edge. Here, transitions are time-independent, but our model allows for time-dependent transitions as well. }
	\label{fig:figure_mdpsl}
\end{figure}

\section{Strategy Synthesis under Uncertain Navigation Times}
\label{sec:problem-mdp}
 
Consider now that disturbances both vary over time and follow some (known) stochastic distribution. For instance, in an office-like environment, access to the kitchen between noon and 12:30 p.m. might take 10 minutes half of the time, and 20 minutes the other half of the time. To represent this type of uncertainty, we now consider strategy synthesis over finite-state, labeled MDPs  \cite{puterman2014markov}. We define a labeled MDP by $\mathcal{M}=(S, A, P, T, \Pi, L)$. As in the VWTS, $S$ is a finite set of states, and $s_0$ is an initial state. Then, $A$ is the set of actions. In this setting, actions describe either navigation from one location to another or waiting. Thus, there is one action per edge in the VWTS, all of which are represented by $A=\{a_{{s_1},{s_2}}| \forall (s_1, s_2)\in \delta \}$. The transition function $P$ encodes the stochastic delays. Formally, we define $P:S\times [0,T] \times A \times S \times [0,T]\rightarrow [0,1]$
where $P(s, t, a, s', t')=p$ is the probability that, if we execute action $a$ at time $t$ in state $s$, we will arrive in state $s'$ at time $t'$. Finally, as in the VWTS,  $T$ is the finite time horizon, $\Pi$ is the set of atomic predicates, and L is the labeling function. Fig. \ref{fig:figure_mdpsl} displays an example of a three-state MDP. 
A \textit{stochastic strategy }in the MDP is defined by $\mu:S\times[0,T]\times A\rightarrow [0,1]$. Then, $\mu(s,t,a)=p$ means that when the robot is at state $s$ at time $t$, it should take action $a$ with probability $p$.

\subsection{Problem Definition and Approach}

\begin{problem}
Given an MDP $\mathcal{M}$ and a set of tasks $D$ expressed in MITL,  we would like to find a \textit{strategy} $\mu$ that maximizes the \textit{expected} total sum of the tasks' temporal robustnesses, weighted by their respective priorities, i.e.,
\begin{subequations}\label{eq:optimisation_mdp}
	\begin{align}
			\text{max}_\mu E_\mu\left[ \sum_{(\phi_i,p_i) \in D} \eta_{\phi_i}(\sigma(S_\mu))  \cdot p_i\right].
	\end{align}
 
\end{subequations}
\end{problem}

In order to succinctly represent Problem 2 with Mixed-integer linear programming, we need to define an extended MDP $\mathcal{\tilde{M}}=(\tilde{S}, A, \tilde{P}, T, \Pi, \tilde{L})$, which is augmented to include \textit{trace history} inside the state space. Formally, $\tilde{S}=\bigcup_{t\in[0,T]} (S\times [0,T])^t$, and denotes not only the current state, but all previous states the robot has visited. For example, $(s_1,t_1), (s_2,t_2), (s_3,t_3) \in \bigcup_{t\in[0,T]} (S\times [0,T])^t$ means that the robot visited state $s_1$ at time $t_1$, $s_2$ at time $t_2$, and is currently at $s_3$ at time $t_3$. For the initial state $s_0\in S$, $(s_0,0)\in \tilde{S}$ is the initial state in $\tilde{M}$. Then, $\tilde{P}:\tilde{S}, A, \tilde{S}'\rightarrow [0,1]$ is defined with respect to $P$. For $\tilde{s} = (s_1,t_1), (s_2, t_2,), ...,  (s_n, t_n)$, if $P(s_n, t_n, A, s', t')=p$, then $P(\tilde{s}, a, \tilde{s}')=p$  for $\tilde{s}'=(s_1,t_1), (s_2,t_2), ...,  (s_n, t_n), (s', t')$. Finally, $\tilde{L}:\tilde{S}\rightarrow 2^\Pi$ is defined by $\tilde{L}(\tilde{s})=L(s_n)$, for arbitrary $\tilde{s} = (s_1,t_1), (s_2, t_2,), ...,  (s_n, t_n)$. Fig. \ref{fig:figure_mdps_unrolled} displays the reachable states of an MDP with trace history.

\begin{figure}[t] 
	\centering
	\includegraphics[width=0.99\linewidth]{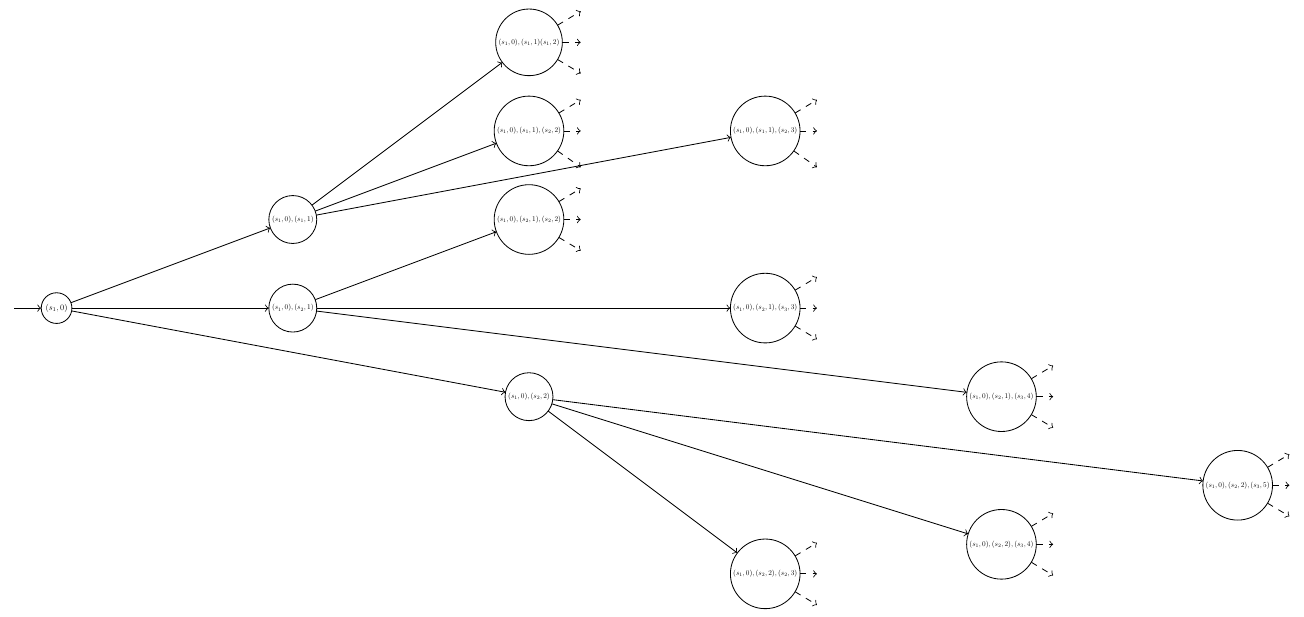}
	\caption{Reachable states of the MDP from Fig. \ref{fig:figure_mdpsl} after three transitions.}
	\label{fig:figure_mdps_unrolled}
\end{figure}

Now, we define a 2-stage Mixed-integer linear programming problem to synthesize optimal strategies with respect to \eqref{eq:optimisation_mdp}. 

\subsubsection{A reward function for MITL temporal robustness}
First, we define the reward function $\tilde{R}:\tilde{S}\rightarrow \mathbb{R}$ of our MDP $\tilde{M}$. Consider an arbitrary state  $\tilde{s}=(s_1,t_1), (s_2,t_2), ...,  (s_m, t_n)\in \tilde{S}$. If $t_n=T$, then $\tilde{R}(\tilde{s})= \sum_{(\phi_i,p_i) \in D} \eta_{\phi_i}(\sigma(\tilde{s}))  \cdot p_i$, where $\sigma$ is defined as in Section~\ref{sec:problem-deterministic}. This reward is explicitly calculated for each possible $\tilde{s}$ such that $t_n= T$ through a MILP, as in \cite{rodionova2022temporal}. For any $\tilde{s}$ such that $t_n\neq T$, $\tilde{R}(\tilde{s})= 0$. This reward structure is designed so that a reward is accumulated only when an agent completes their full execution over the time horizon $T$, and thus, their full trace is available to compute the temporal robustness.

\subsubsection{Solving the MDP}
Next, we synthesize an optimal strategy $\mu$ by solving a Linear Program (LP) as in \cite{puterman2014markov}. In practice, other MDP solvers can be used at this stage.  \cite{puterman2014markov} introduces a continuous variable $0\leq o^{h}_{\tilde{s},a}\leq 1$ for each state $\tilde{s} \in \mathcal{\tilde{S}}$, for each action $a \in A$, and for each planning-step $h \in [0,T]$, where $h$ represents the current number of transitions. Note that this is distinct from the current timestep, which is kept track of in the state $\tilde{s}$.  Each variable $o^{h}_{\tilde{s},a}$ stands for the occupancy measure of the history-dependent MDP state: it represents the probability that, for a given synthesized strategy, the robot occupies state  $\tilde{s}$ at planning step $h$ and takes action $a$. Now, the LP is:

	\begin{subequations}\label{eq:main_problem_mdp}
		\begin{align}
			\max_{\mathbf{\mu}} ~& \sum_{a\in A}\sum_{{\tilde{s}\in \tilde{S}}}o^{T}_{\tilde{s},a}\tilde{R}(\tilde{s}) \label{eq:main_problem_cost_mdp}\\
			\mathrm{s.t.~} & \sum_{a \in A} o^{0}_{(s_0,0),a} = 1 \label{eq:mdp_s0} \\
			\begin{split}
			& \sum_{a \in A} o^{h+1}_{\tilde{s}',a'} =  \sum_{\tilde{s} \in \tilde{S}}  \sum_{a \in A} o^{h}_{\tilde{s},a} \tilde{P}(\tilde{s}, a, \tilde{s}')
			\\ &  \quad \forall \tilde{s}' \in \mathcal{\tilde{S}}, \forall h \in [0,T-1]
		\end{split} \label{eq:mdp_thsa_proba}	
				\end{align}
	\end{subequations}

Constraint \eqref{eq:mdp_s0} ensures that the strategy occupies the initial state at $h=0$ with probability 1. Then, \eqref{eq:mdp_thsa_proba} encodes the transition probabilities.
Once the LP is solved, the optimal strategy can be constructed from the occupancy measures. 

\subsection{Planning with a receding horizon}
\vspace{-0.05cm}
When planning over history-dependent MDPs, the state space of $\tilde{M}$ is exponential in the time horizon $T$. As a result, it may be desirable to plan for a receding horizon $T_r<T$, and re-plan throughout execution. We now present a method of receding horizon planning that preserves guarantees over temporal robustness via a worst-case lookahead. We describe this process for right temporal robustness, but it can be modified to left temporal robustness. 

The receding horizon method we present here exclusively modifies the reward function of our MDP, which is calculated explicitly before solving the MDP. Before describing this modified reward function, we first need to describe how to build a worst-case VWTS ($\bar{W}$) from any MDP. $\bar{W}$ describes the maximum possible delays that could occur if the uncertainty in the MDP was realized in an adversarial manner. For example, if at time $t$, transitioning from $s_1$ to $s_2$ takes time 1 with  probability  0.5  and takes time 2  with  probability  0.5, the worst-case VWTS will take time 2  to transition from  $s_1$ to $s_2$  at time $t$.
Therefore, for a given MDP $\mathcal{M}=(S, A, P, T, \Pi, L)$, we build a worst-case VWTS $\bar{W}=(S, s_0, \delta, \Pi, L, \Delta)$ where $S, s_0, \Pi, L$ remain the same and $\delta$, $\Delta$ are defined as: $(s_1,s_2)\in\delta$ if and only if there exists some $h_1, h_2\in [0,T], a\in A$ such that $P(s_1, h_1, a, s_2, h_2)>0$.
Then for every $(s_1,s_2)\in\delta, h_1\in[0,T]$, we define:
\vspace{-0.15cm}
\begin{align*}
  	\Delta&(s_1, h_1, s_2)\\
  	&=\text{max}\{h_2|  h_2\in [0,T], a\in A, P(s_1, h_1, a, s_2, h_2)>0\}.
\end{align*}
 
We can now redefine the reward function used in  \eqref{eq:main_problem_cost_mdp}. For a given history $\tilde{s}\in \tilde{S}$ for horizon $T_r$, we can extend $\tilde{s}$ to include the optimal trace in the worst-case VWTS. To do this, we define a new WTS,  $\bar{W}_{\tilde{s}}=(S, s_0, \delta, \Pi, L, \Delta_{\tilde{s}})$, where $S, s_0, \delta, \Pi, L$ are the same as in $\bar{W}$, but the weight function~$\Delta$ is modified to ensure that the only feasible trace through the transition system before time $T_r$ is $\tilde{s}$. Formally:
 \begin{enumerate}
 	\item if $h<T_r$, $s_1, h, s_2\in \tilde{s}$, then $\Delta_{\tilde{s}}(s_1, h, s_2)$ is equal to the realized time delay between $s_1$ and $s_2$ in trace $\tilde{s}$, 
  \item if $h<T_r$, $s_1, h, s_2\nin \tilde{s}$, then $\Delta_{\tilde{s}}(s_1, h, s_2)=\infty$, and 
 	\item if $h>=T_r$, $\Delta_{\tilde{s}}(s_1, h, s_2)=\Delta_{\tilde{s}}(s_1, h, s_2).$
 \end{enumerate}
We now use $\bar{W}_{\tilde{s}}$ to define a new, worst-case lookahead reward function by $\bar{R}:\tilde{S}\rightarrow \mathbb{R}$, where $\bar{R}(\tilde{s})$ is equal to the temporal robustness of $\bar{W}_{\tilde{s}}$ as calculated by Section~\ref{sec:problem-deterministic}.

\begin{lemma}
\label{lemma:worst_case}
Solving the LP with $\bar{R}(\tilde{s})$ returns the worst-case temporal robustness that could be achieved after any number of future, receding horizon replannings.
\end{lemma}
\begin{proof} (Sketch)
For notational convenience, we show this for one replanning, but the results extend to multiple replannings. 
The final strategy can be described by $\mu=\mu_{[0,T_r]}+\mu_{[T_r, T]}$ where   $\mu_{[0,T_r]}$ describes how to plan until $T_r$ and $\mu_{[T_r, T]}$ describes how to plan after. Strategy $\mu$ will have expected  temporal robustness described by:
\begin{equation}
\sum_{a\in A}\sum_{{\tilde{s}\in \tilde{S}}}o^{T}_{\tilde{s},a}\tilde{R}(\tilde{s}).
\end{equation}
We can partition the $T$-length elements of $\tilde{S}$ into sets that share a prefix of length $T_r$. Let $\tilde{S}_{T}$, $\tilde{S}_{T_r}$ be the set of $T$-length and $T_r$-length prefixes, respectively. Then we can equivalently define expected temporal robustness for $\mu$ as:
\begin{equation}
	\sum_{a\in A}\sum_{{\tilde{s}_{T_r}\in \tilde{S}_{T_r}}}\sum_{\tilde{s}_{T}\in \tilde{S}_{T}\ni\tilde{s}_{T_r}\\\in\tilde{s}_{T} } o^{T}_{\tilde{s}_{T},a}\tilde{R}(\tilde{s}_{T}).
\end{equation}
But, for any $\tilde{s}_{T}\in \tilde{S}_{T}\ni\tilde{s}_{T_r}$, $\tilde{R}(\tilde{s}_{T})\geq \bar{R}(\tilde{s}_{T_r})$, because the prefix for $\tilde{s}_{T}$ is $\tilde{s}_{T_r}$, and $\bar{R}$ is a worst-case temporal robustness for  $\tilde{s}_{T_r}$.  Thus, the expected temporal robustness for strategy $\mu$ is at least:
\begin{align}
\sum_{a\in A}&\sum_{{\tilde{s}_{T_r}\in \tilde{S}_{T_r}}}\sum_{\tilde{s}_{T}\in \tilde{S}_{T}\ni\tilde{s}_{T_r}\in\tilde{s}_{T} }o^{T}_{\tilde{s}_{T},a}\bar{R}(\tilde{s}_{T_r})
\\ & = \sum_{a\in A}\sum_{{\tilde{s}_{T_r}\in \tilde{S}_{T_r}}} \bar{R}(\tilde{s}_{T_r} )\sum_{\tilde{s}_{T}\in \tilde{S}_{T}\ni\tilde{s}_{T_r}\in\tilde{s}_{T} } o^{T}_{\tilde{s}_{T},a}
\end{align}

Finally, we note that $\sum_{\tilde{s}_{T}\in \tilde{S}_{T}\ni\tilde{s}_{T_r}\\\in\tilde{s}_{T} } o^{T}_{\tilde{s}_{T},a}$ is exactly equal to $ o^{T_r}_{\tilde{s}_{T_r},a}$, because $\{\tilde{s}_{T}\in \tilde{S}_{T}\ni\tilde{s}_{T_r}\in\tilde{s}_{T} \}$ contains all possible extension of $\tilde{s}_{T_r}$. Finally, we conclude that the temporal robustness for the final strategy $\mu$ is at least as much as the temporal robustness determined before pre-planning. 
\end{proof}

\section{Experiments}
\label{sec:experiments}
\begin{figure}[t]
    \centering
    \includegraphics[width=0.8\linewidth]{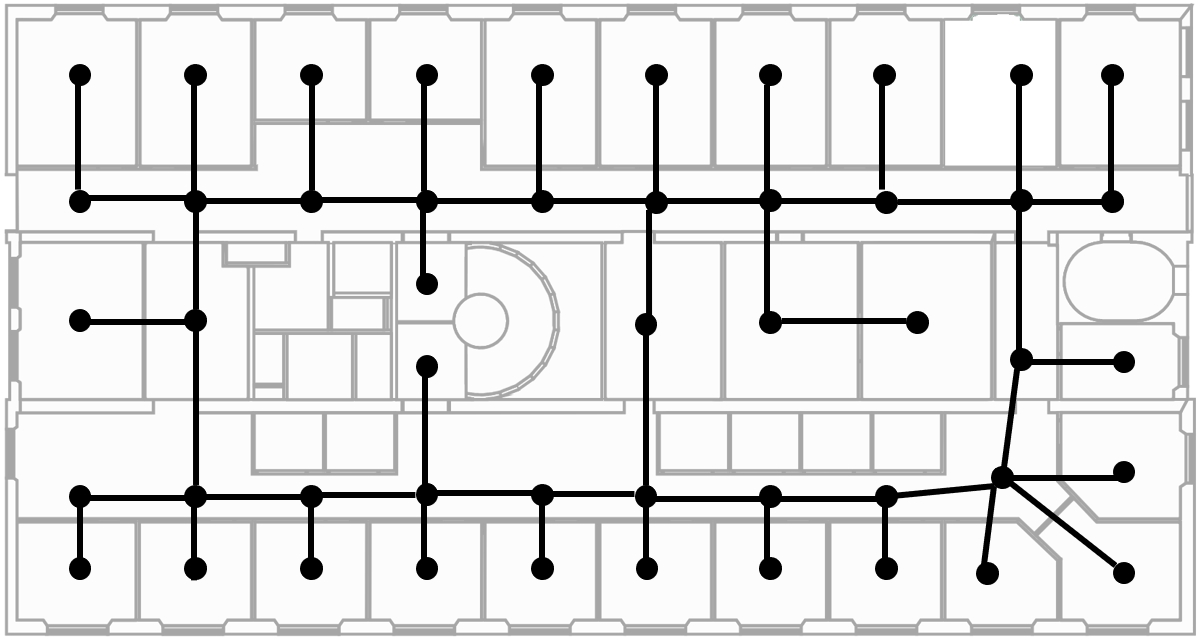}
    \caption{Transition system of an office-like environment. Labels and self-transitions are omitted for readability. Transitions are bi-directional. For more details, see our repository\footref{gh_repo}.}
    \label{fig:lv24}
\end{figure}
We implemented and tested our methods in Python 3.8, using Gurobi \cite{gurobi} for the LP solving\footnote{\label{gh_repo}\href{https://github.com/KTH-RPL-Planiacs/mitl_task_solver_temporal_robustness}{https://github.com/KTH-RPL-Planiacs/mitl\_task\_solver\_temporal\_robustness}}.
We ran simulations on an Intel i7-8665U CPU and 32GB RAM, with a timeout of 30000 seconds.

\subsection{Deterministic Navigation Times} 
We ran experiments varying the number of states, the number of tasks, and the horizon on a VWTS with transitions from the office-like environment described in Fig. \ref{fig:lv24}. Table~\ref{tab:res_simulation_wts} shows that, as the time horizon $T$ and the number of tasks $D$ increase, the time it takes to solve the MILP and yield a strategy increases.  Our results demonstrate the scalability of the \textit{deterministic navigation times} scenario. One could achieve day-long task planning even with time steps of 1 minute. 
One can easily adjust the time resolution and the planning time window.

\begin{table}[t]
\centering
\caption{No. of tasks ($|D|$); No. of states in the VWTS ($|S|$); time horizon ($T$); encoding time of the LP constraints ($t_\textrm{encoding}$); solving time by the solver ($t_\textrm{solving}$); No. of LP variables (\#$\textrm{lpvars}$); and No. of LP constraints (\#$\textrm{lpconst}$).}
\begin{tabular}{c|c|c|c|c|c|c}
\hline
$|S|$ & $|D|$ & $T$ & $t_\textrm{encoding}$ (s) & $t_\textrm{solving}$ (s) & \#$\textrm{lpvars}$ & \#$\textrm{lpconst}$ \\ \hline
\multirow{12}{*}{46} & \multirow{4}{*}{5} & 50 & 0.138 & 0.408 & 4037 & 7543 \\ \cline{3-7} 
 &  & 100 & 0.268 & 6.353 & 6943 & 14559 \\ \cline{3-7} 
 &  & 500 & 1.181 & 29.79 & 25851 & 61903 \\ \cline{3-7} 
 &  & 1000 & 2.049 & 44.03 & 49146 & 104683 \\ \cline{2-7} 
 & \multirow{4}{*}{10} & 50 & 0.204 & 0.267 & 4845 & 9929 \\ \cline{3-7} 
 &  & 100 & 0.354 & 40.08 & 7803 & 18023 \\ \cline{3-7} 
 &  & 500 & 1.453 & 50.00 & 27736 & 73563 \\ \cline{3-7} 
 &  & 1000 & 2.396 & 298.6 & 51381 & 118643 \\ \cline{2-7}
 & \multirow{4}{*}{20} & 50 & 0.579 & 0.619 & 7076 & 27427 \\ \cline{3-7} 
 &  & 100 & 0.973 & 293.2 & 10876 & 47383 \\ \cline{3-7} 
 &  & 500 & 4.153 & 6123 & 33068 & 204683 \\ \cline{3-7} 
 &  & 1000 & 4.870 & 13091 & 56937 & 233051 \\ \hline
 \multirow{12}{*}{92} & \multirow{4}{*}{5} & 50 & 0.254 & 0.404 & 7274 & 12981 \\ \cline{3-7} 
 &  & 100 & 0.479 & 18.23 & 12198 & 23581 \\ \cline{3-7} 
 &  & 500 & 2.127 & 33.88 & 49723 & 102241 \\ \cline{3-7} 
 &  & 1000 & 4.129 & 37.82 & 96063 & 195121 \\ \cline{2-7} 
 & \multirow{4}{*}{10} & 50 & 0.414 & 0.765 & 8082 & 15367 \\ \cline{3-7} 
 &  & 100 & 0.588 & 14.04 & 12993 & 26951 \\ \cline{3-7} 
 &  & 500 & 2.246 & 23.63 & 50188 & 104941 \\ \cline{3-7} 
 &  & 1000 & 4.525 & 91.30 & 98358 & 208229 \\ \cline{2-7}
 & \multirow{4}{*}{20} & 50 & 0.827 & 1.629 & 10313 & 32865 \\ \cline{3-7} 
 &  & 100 & 1.314 & 23723 & 16492 & 52193 \\ \cline{3-7} 
 &  & 500 & 4.019 & $>$30000 & 56361 & 167121 \\ \cline{3-7} 
 &  & 1000 & 6.037 & $>$30000 & 102356 & 259121 \\ \hline
\end{tabular}
\label{tab:res_simulation_wts}
\end{table}

\subsection{Uncertain Navigation Times} 
We ran experiments varying the number of states, the number of tasks, the horizon, and the receding horizon on a uncertain navigation time MDP with transitions from the office-like environment described in Fig. \ref{fig:lv24}. The results in Table~\ref{tab:res_simulation_mdp} show that the time to encode the MDP is significantly more cumbersome than the time to solve the MDP. The encoding time includes explicitly defining the reward function, which relies on solving a series of calls to the VWTS MITL temporal robustness MILP. As the receding horizon $T_r$ increases, the number of calls to that MILP increases exponentially. As the time horizon  $T$ increases and the number of tasks $D$ increases, the time it takes to execute that MILP increases.

\begin{table}[t]
\centering
\caption{No. of tasks ($|D|$); No. of states in the MDP ($|S|$); time horizon ($T$); receding horizon ($T_r$); encoding time of LP constraints ($t_\textrm{encoding}$); solving time by solver ($t_\textrm{solving}$); No. of LP variables (\#$\textrm{lpvars}$); and No. of LP constraints (\#$\textrm{lpconst}$).}
\begin{tabular}{c|c|c|c|c|c|c|c}
\hline
$|S|$ & $|D|$ & $T$ & $T_r$ & $t_\textrm{encoding}$ (s) & $t_\textrm{solving}$ (s) & \#$\textrm{lpvars}$ & \#$\textrm{lpconst}$  \\ \hline
\multirow{8}{*}{46} & \multirow{4}{*}{2} &  \multirow{2}{*}{25} & 5 & 28.86 & 0.054 & 1186 & 161 \\ \cline{4-8}
 & & & 7 & 479.6 & 0.364 & 14596 & 1964 \\ \cline{3-8}
 & & \multirow{2}{*}{50} & 5 & 255.7 & 0.055 & 1186 & 161 \\ \cline{4-8}
 & & & 7 & 2271 & 0.317 & 14596 & 1964 \\ \cline{2-8}
 & \multirow{4}{*}{5} & \multirow{2}{*}{25} & 5 & 36.08 & 0.037 & 1186 & 161 \\ \cline{4-8}
 & & & 7 & 578.0 & 0.379 & 14596 & 1964 \\ \cline{3-8}
 & & \multirow{2}{*}{50} & 5 & 70.66 & 0.085 & 1186 & 161 \\ \cline{4-8}
 & & & 7 & 906.7 & 0.342 & 14596 & 1964 \\ \hline
\end{tabular}
\label{tab:res_simulation_mdp}
\end{table}

\section{Conclusions and Future Work}
\label{sec:conclusion}

We developed a planning methodology for optimizing MITL temporal robustness in scenarios where robot navigation times are uncertain. 
For real-world applications, the algorithm considering the MDP model finds its most beneficial use within relatively short planning horizons. This approach is particularly suited for situations where frequent replanning with fewer task demands is necessary.
While our study primarily focused on a single robot navigating an office-like environment with multiple tasks and uncertain navigation times, we believe our method has broader applicability. In future, we aim to extend it to scenarios involving robots operating in dynamic environments, such as navigation with varying levels of crowds, or subject to travel delays induced by human activity \cite{alami2006toward}. We aim to learn a spatio-temporal model of human occupation from temporal sequences \cite{patel2022proactive}. We would use such an activity distribution model, as well as people's typical schedules for daily activities, as input to our planner. In these cases, the stochastic nature of navigation durations becomes crucial, as the robot must adapt to unpredictable delays during task execution. Additionally, we aim to extend our strategy synthesis to scenarios involving multiple robots.
Finally, we believe that temporal robustness is a good measure of a strategy's resilience in accommodating time shifts while still satisfying temporal task requirements. It provides a quantitative measure of adaptability, particularly important in addressing navigational challenges where uncertain navigation times can affect a robot's ability to meet task deadlines and efficiently prioritize tasks. Furthermore, the optimization function can be customized to linear functions, depending on specific optimization objectives and applications.

\section*{Acknowledgement}
We would like to thank Patric Jensfelt for his support and valuable discussions.

\newpage

\bibliographystyle{IEEEtran}
\bibliography{biblio}

\begin{thebibliography}{10}
\providecommand{\url}[1]{#1}
\csname url@rmstyle\endcsname
\providecommand{\newblock}{\relax}
\providecommand{\bibinfo}[2]{#2}
\providecommand\BIBentrySTDinterwordspacing{\spaceskip=0pt\relax}
\providecommand\BIBentryALTinterwordstretchfactor{4}
\providecommand\BIBentryALTinterwordspacing{\spaceskip=\fontdimen2\font plus
\BIBentryALTinterwordstretchfactor\fontdimen3\font minus
  \fontdimen4\font\relax}
\providecommand\BIBforeignlanguage[2]{{%
\expandafter\ifx\csname l@#1\endcsname\relax
\typeout{** WARNING: IEEEtran.bst: No hyphenation pattern has been}%
\typeout{** loaded for the language `#1'. Using the pattern for}%
\typeout{** the default language instead.}%
\else
\language=\csname l@#1\endcsname
\fi
#2}}

\bibitem{koymans1990specifying}
R.~Koymans, ``Specifying real-time properties with metric temporal logic,''
  \emph{Real-time systems}, vol.~2, no.~4, pp. 255--299, 1990.

\bibitem{fainekos_2005}
G.~Fainekos, H.~Kress-Gazit, and G.~Pappas, ``Temporal logic motion planning
  for mobile robots,'' in \emph{Proceedings of the 2005 IEEE International
  Conference on Robotics and Automation}, 2005, pp. 2020--2025.

\bibitem{kressgazit_2009}
H.~Kress-Gazit, G.~E. Fainekos, and G.~J. Pappas, ``Temporal-logic-based
  reactive mission and motion planning,'' \emph{IEEE Transactions on Robotics},
  vol.~25, no.~6, pp. 1370--1381, 2009.

\bibitem{lahijanian_2012}
M.~Lahijanian, S.~B. Andersson, and C.~Belta, ``Temporal logic motion planning
  and control with probabilistic satisfaction guarantees,'' \emph{IEEE
  Transactions on Robotics}, vol.~28, no.~2, pp. 396--409, 2012.

\bibitem{raman_2014}
V.~Raman, A.~Donzé, M.~Maasoumy, R.~M. Murray, A.~Sangiovanni-Vincentelli, and
  S.~A. Seshia, ``Model predictive control with signal temporal logic
  specifications,'' in \emph{53rd IEEE Conference on Decision and Control},
  2014, pp. 81--87.

\bibitem{vasile2017minimum}
C.-I. Vasile, J.~Tumova, S.~Karaman, C.~Belta, and D.~Rus, ``Minimum-violation
  scltl motion planning for mobility-on-demand,'' in \emph{2017 IEEE
  International Conference on Robotics and Automation (ICRA)}.\hskip 1em plus
  0.5em minus 0.4em\relax IEEE, 2017, pp. 1481--1488.

\bibitem{liang2022fair}
K.~Liang and C.-I. Vasile, ``Fair planning for mobility-on-demand with temporal
  logic requests,'' in \emph{2022 IEEE/RSJ International Conference on
  Intelligent Robots and Systems (IROS)}.\hskip 1em plus 0.5em minus
  0.4em\relax IEEE, 2022, pp. 1283--1289.

\bibitem{cardona2022planning}
G.~A. Cardona, D.~Salda{\~n}a, and C.-I. Vasile, ``Planning for modular aerial
  robotic tools with temporal logic constraints,'' in \emph{2022 IEEE 61st
  Conference on Decision and Control (CDC)}.\hskip 1em plus 0.5em minus
  0.4em\relax IEEE, 2022, pp. 2878--2883.

\bibitem{rodionova2022temporal}
A.~Rodionova, L.~Lindemann, M.~Morari, and G.~Pappas, ``Temporal robustness of
  temporal logic specifications: Analysis and control design,'' \emph{ACM
  Transactions on Embedded Computing Systems}, vol.~22, no.~1, pp. 1--44, 2022.

\bibitem{kurtz2021more}
V.~Kurtz and H.~Lin, ``A more scalable mixed-integer encoding for metric
  temporal logic,'' \emph{IEEE Control Systems Letters}, vol.~6, pp.
  1718--1723, 2021.

\bibitem{tumova2016least}
J.~Tumova, S.~Karaman, C.~Belta, and D.~Rus, ``Least-violating planning in road
  networks from temporal logic specifications,'' in \emph{2016 ACM/IEEE 7th
  International Conference on Cyber-Physical Systems (ICCPS)}.\hskip 1em plus
  0.5em minus 0.4em\relax IEEE, 2016, pp. 1--9.

\bibitem{donze_robust_stl}
A.~Donz\'e and O.~Maler, ``Robust satisfaction of temporal logic over
  real-valued signals,'' in \emph{Proceedings of the International Conference
  on Formal Modeling and Analysis of Timed Systems}, 2010.

\bibitem{puterman2014markov}
M.~L. Puterman, \emph{Markov decision processes: discrete stochastic dynamic
  programming}.\hskip 1em plus 0.5em minus 0.4em\relax John Wiley \& Sons,
  2014.

\bibitem{gurobi}
\BIBentryALTinterwordspacing
L.~Gurobi~Optimization, ``Gurobi optimizer reference manual,'' 2020. [Online].
  Available: \url{http://www.gurobi.com}
\BIBentrySTDinterwordspacing

\bibitem{alami2006toward}
R.~Alami, A.~Clodic, V.~Montreuil, E.~A. Sisbot, and R.~Chatila, ``Toward
  human-aware robot task planning.'' in \emph{AAAI spring symposium: to boldly
  go where no human-robot team has gone before}, 2006, pp. 39--46.

\bibitem{patel2022proactive}
M.~Patel and S.~Chernova, ``Proactive robot assistance via spatio-temporal
  object modeling,'' \emph{arXiv preprint arXiv:2211.15501}, 2022.

\end{thebibliography}

\end{document}